\documentclass[11pt, reqno]{amsart}
\usepackage[margin=3.6cm]{geometry}
\usepackage{hyperref}
\usepackage{array,booktabs,tabularx}
\usepackage{graphicx}
\usepackage{amssymb}
\usepackage{enumerate}
\usepackage{color}
\usepackage{esint}
\usepackage{mathtools}
\usepackage{dsfont}
\usepackage{ esint }
\usepackage{enumitem}

\newtheorem{theorem}{Theorem}

\newtheorem{proposition}[theorem]{Proposition}
\newtheorem{lemma}[theorem]{Lemma}

\newtheoremstyle{named}{}{}{\itshape}{}{\bfseries}{.}{.5em}{\thmnote{#3 }#1} \theoremstyle{named} 

\theoremstyle{definition}

\newcommand{\R}{{\mathbb R}}

\newcommand{\ep}{\varepsilon}

\newcommand{\gives}{\ensuremath{\rightarrow}}

\newcommand{\setst}[2]{\ensuremath{ \{ #1\,|\,#2 \}}}

\newcommand{\abs}[1]{\ensuremath{\left| #1 \right|}}
\newcommand{\lr}[1]{\ensuremath{\left(#1 \right)}}
\newcommand{\norm}[1]{\left\lVert#1\right\rVert}
\newcommand{\inprod}[2]{\ensuremath{\left\langle#1,#2\right\rangle}}

\newcommand{\twiddle}[1]{\ensuremath{\widetilde{#1}}}

\newcommand{\w}{\omega}

\newcommand{\set}[1]{\ensuremath{\{#1\}}}

\def\XXint#1#2#3{{\setbox0=\hbox{$#1{#2#3}{\int}$} \vcenter{\hbox{$#2#3$}}\kern-.5\wd0}}

\DeclareMathOperator{\Relu}{ReLU}

\title[]{Universal Function Approximation by Deep Neural Nets with
  Bounded Width and ReLU Activations}

\author[B. Hanin]{Boris Hanin}

\address[B. Hanin]{Department of Mathematics, Texas A\&M, College Station,
  United States\medskip}
\email{bhanin@math.tamu.edu}

\begin{document}
\maketitle
\begin{abstract}
  This article concerns the expressive power of depth in neural nets with
  $\Relu$ activations and bounded width. We are particularly
  interested in the following questions: what is the minimal width
  $w_{\text{min}}(d)$ so that $\Relu$ nets of width
  $w_{\text{min}}(d)$ (and arbitrary depth) can approximate any
  continuous function on the unit cube $[0,1]^d$ aribitrarily well?
  For $\Relu$ nets near this minimal width, what can one say about the depth
  necessary to approximate a given function? We obtain an essentially
  complete answer to these questions for convex functions. 
  Our approach is based on the observation that, due to the convexity
  of the $\Relu$ activation, $\Relu$ nets
  are particularly well-suited for representing convex
  functions. In particular, we prove that $\Relu$ nets with width $d+1$ can 
  approximate any continuous convex function of $d$ variables arbitrarily
  well. These results then give quantitative depth estimates for the rate of
  approximation of any continuous scalar function on the $d$-dimensional
  cube $[0,1]^d$ by $\Relu$ nets with width $d+3.$
\end{abstract}

\section{Introduction}
Over the past several years, neural nets $-$ particularly deep nets $-$ have become the state of the art in
a remarkable number of machine learning problems, from mastering Go
 to image recognition/segmentation and machine translation (see the review
 article \cite{BHL} for more background). Despite all their
practical successes, a robust theory of why they
work so well is in its infancy. Much of the work to date has focused on
the problem of explaining and quantifying the \textit{expressivity}
$-$ the ability to approximate a rich class of functions $-$ of deep neural
nets \cite{ABM, LMP, LRT, MhPo, PLR, RPK, Tel1, Tel2, Tel3, Yar}. Expressivity can be
seen both as an effect of both depth and width. It
has been known since at least the work of Cybenko \cite{Cyb} and
Hornik-Stinchcombe-White \cite{HSW} that if no constraint is placed on
the width of a hidden layer, then a single hidden layer is enough to
approximate essentially any function. The purpose of this article, in
contrast, is to investigate the ``effect of depth without the aid of 
width.'' More precisely, for each $d\geq 1$ we would like to estimate
\begin{equation}
  w_{\text{min}}(d):=\min\left\{w\in \mathbb N
    \,\bigg|\, \begin{array}{c} \Relu \text{nets of width }w \text{ 
      can approximate any}\\ \text{positive continuous function on
               }[0,1]^d \text{ arbitrarily well} \end{array}\right\}.
\end{equation}
In Theorem \ref{T:lip-approx}, we prove that
$\w_{\text{min}}(d)\leq d+2.$ This raises two questions: 
\begin{enumerate}
\item[{\bf Q1.}] Is the estimate in the previous line sharp? \medskip
\item[{\bf Q2.}] How efficiently can $\Relu$ nets of a given width $w\geq
  w_{\text{min}}(d)$ approximate a given continuous function of $d$
  variables? \bigskip 
\end{enumerate}

On the subject of Q1, we will prove in forthcoming work with M. Sellke \cite{HaSe}
that in fact $\w_{\text{min}}(d)=d+1.$ When $d=1$, the lower bound is simple to check, and the upper bound
follows for example from Theorem 3.1 in
\cite{MhPo}. The main results in this article, however, concern Q1 and
Q2 for convex functions. For instance, we prove in Theorem
\ref{T:lip-approx} that
\begin{equation}\label{E:conv-width}
  w_{\text{min}}^{\text{conv}}(d)\leq d+1,
\end{equation}
where
\begin{equation}
  w_{\text{min}}^{\text{conv}}(d):=\min\left\{w\in \mathbb N
   \,\bigg|\, \begin{array}{c} \Relu \text{nets of width }w \text{ 
      can approximate any}\\ \text{positive convex function on
               }[0,1]^d\text{ arbitrarily well} \end{array}\right\}.
\end{equation}
This illustrates a central point of the present paper: the
convexity of the $\Relu$ activation makes $\Relu$ nets well-adapted to
representing convex functions on $[0,1]^d.$

 Theorem \ref{T:lip-approx} also addresses Q2 by providing quantitative
estimates on the depth of a $\Relu$ net with width
$d+1$ that approximates a given convex function. We provide similar depth estimates for 
arbitrary continuous functions on $[0,1]^d,$ but this time for nets of width
$d+3.$ Several of our depth estimates are based on the work of 
Bal\'{a}zs-Gy\"{o}rgy-Szepesv\'{a}ri \cite{BGS} on max-affine
estimators in convex regression.  

In order to prove Theorem \ref{T:lip-approx}, we must understand what
functions can be exactly computed by a $\Relu$ net. Such functions are always piecewise
affine, and we prove in Theorem 
\ref{T:max-affine-approx} the converse: every piecewise affine
function on $[0,1]^d$ can be \textit{exactly} represented by a $\Relu$
net with hidden layer 
width at most $d+3$. Moreover, we prove that the depth of
the network that computes such a function is bounded by the
number affine pieces it contains. This extends the results
of Arora-Basu-Mianjy-Mukherjee (e.g. Theorem 2.1 and
Corollary 2.2 in \cite{ABM}). 

Convex functions again play a special role. We show that every
convex function on $[0,1]^d$ that is piecewise affine with $N$ pieces can be 
represented exactly by a $\Relu$ net with width $d+1$ and depth $N.$

\section{Statement of Results}\label{S:results}
To state our results precisely, we set notation and recall
several definitions. For $d\geq 1$ and a
continuous function $f:[0,1]^d \gives \R,$ write 
\[\norm{f}_{C^0}:=\sup_{x\in [0,1]^d}\abs{f(x)}.\]
Further, denote by 
\[\w_f(\ep):=\sup\setst{\abs{f(x)-f(y)}}{\abs{x-y}\leq \ep}\]
the modulus of continuity of $f,$ whose value at $\ep$ is the maximum
 $f$ changes when its argument moves by at most $\ep.$ Note that by
 definition of a continuous function, $\w_f(\ep)\gives 0$ as $\ep\gives 0.$
Next, given $d_{\text{in}}, d_{\text{out}},$ and $ w\geq 1,$ we define a feed-forward neural net with ReLU
activations, input dimension $d_{\text{in}}$, hidden layer width
$w$, depth $n,$ and output dimension $d_{\text{out}}$ to be any member of the
finite-dimensional family of functions
\begin{equation}\label{E:relunet-def}
\Relu \circ A_n \circ \cdots \circ\Relu \circ A_1 \circ \Relu \circ A_1
\end{equation}
that map $\R^d$ to $\R_+^{d_{\text{out}}}=\setst{x=\lr{x_1,\ldots, x_{d_{\text{out}}}}\in \R^{d_{\text{out}}}}{x_i\geq 0}.$ In \eqref{E:relunet-def}, 
\[A_j:\R^w\gives \R^w,\,\, j=2,\ldots, n-1,\qquad A_1:\R^{d_{\text{in}}}\gives
\R^w,\,\, A_n:\R^w\gives \R^{d_{\text{out}}} \] 
are affine transformations, and for every $m\geq 1$
\[\Relu(x_1,\ldots, x_m) = \lr{\max\set{0,x_1},\ldots,
  \max\set{0,x_m}}.\]
We often denote such a net by $\mathcal N$ and write 
\[f_{\mathcal N}(x):=\Relu\circ A_n \circ \cdots \circ\Relu \circ A_1 \circ
\Relu \circ A_1(x)\]
for the function it computes. Our first result contrasts both the
width and depth required to approximate continuous, convex, and smooth
functions by $\Relu$ nets.   

\begin{theorem}\label{T:lip-approx}
Let $d\geq 1$ and $f:[0,1]^d\gives \R_+$ be a positive function with
$\norm{f}_{C^0}=1$. We have the following three cases:
\begin{description}
\item[1. ($f$ is continuous)] There
exists a sequence of feed-forward neural nets  
  $\mathcal N_k$ with ReLU activations, input dimension $d,$ hidden
  layer width $d+2,$ output dimension $1,$ such that
  \begin{equation}\label{E:cont-approx}
    \lim_{k\gives \infty}\norm{f-f_{\mathcal N_k}}_{C^0}=0.
  \end{equation}
In particular, $w_{\text{min}}(d)\leq d+2.$ Moreover, write $\w_f$ for the modulus of
continuity of $f,$ and fix $\ep>0.$ There exists a feed-forward neural nets
  $\mathcal N_\ep$ with ReLU activations, input dimension $d,$ hidden 
  layer width $d+3,$ output dimension $1,$ and  
\begin{equation}\label{E:lip-depth}
\text{depth}\lr{\mathcal N_\ep}= \frac{2\cdot d!}{\w_f(\ep)^d}
\end{equation}
such that
  \begin{equation}
\norm{f-f_{\mathcal N_\ep}}_{C^0}\leq \ep.\label{E:lip-approx}
\end{equation}

\item[2. ($f$ is convex)] There exists a sequence of feed-forward neural nets  
  $\mathcal N_k$ with ReLU activations, input dimension $d,$ hidden
  layer width $d+1,$ and output dimension $1,$ such that
  \begin{equation}\label{E:conv-min}
   \lim_{k\gives \infty} \norm{f-f_{\mathcal N_k}}_{C^0}=0.
  \end{equation}
Hence, $\w_{\text{min}}^{\text{conv}}(d)\leq d+1.$ Further, there
exists $C>0$ such that if $f$ is both convex and
Lipschitz with Lipschitz constant $L,$ then the nets $\mathcal N_k$ in
\eqref{E:conv-min} can be taken to satisfy 
\begin{equation}\label{E:lip-depth-approx}
\text{depth}\lr{\mathcal N_k}= k+1,\qquad \norm{f-f_{\mathcal
    N_k}}_{C^0}\leq C L d^{3/2} k^{-2/d}.
\end{equation}

\item[3. ($f$ is smooth)] There exists a constant $K$ depending only
  on $d$ and a constant $C$ depending only on the maximum of the first
  $K$ derivative of $f$ such that for every $k\geq 3$ the width $d+2$ nets $\mathcal
 N_k$ in \eqref{E:cont-approx} can be chosen so that
\begin{equation}\label{E:smooth-approx}
\text{depth}(\mathcal N_k)=k,\qquad  \norm{f-f_{\mathcal N_k}}_{C^0}\leq C\lr{k-2}^{-1/d}.
\end{equation}
\end{description}

\smallskip

\end{theorem}

The main novelty of Theorem \ref{T:lip-approx} is the width estimate
$w_{\text{min}}^{\text{conv}}(d)\leq d+1$ and the quantitative depth
estimates \eqref{E:lip-depth-approx} for convex functions as well as the
analogous estimates
\eqref{E:lip-depth} and \eqref{E:lip-approx} for continuous
functions. Let us breifly explain the origin of the other estimates. The
relation \eqref{E:cont-approx} and the corresponding estimate 
$w_{\text{min}}(d)\leq d+2$ are a combination of the well-known
fact that $\Relu$ nets with one hidden
layer can approximate any continuous function and a simple
procedure by which a $\Relu$ net with input dimension $d$ and a single
hidden layer of width $n$ can be replaced by another $\Relu$ net that
computes the same function but has depth $n+2$ and width $d+2.$ For
these width $d+2$ nets, we are unaware of how to obtain quantitative
estimates on the depth required to approximate a fixed continuous
function to a given precision. At the expense of changing the
width of our $\Relu$ nets from $d+2$ to $d+3,$ however, we furnish the
estimates \eqref{E:lip-depth} and \eqref{E:lip-approx}. On the other
hand, using Theorem 3.1 in \cite{MhPo}, when $f$ is sufficiently
smooth, we obtain the depth estimates \eqref{E:smooth-approx}
for width $d+2$ $\Relu$ nets. 

Our next result concerns the exact representation of 
piecewise affine functions by $\Relu$ nets. Instead of measuring the
complexity of a such a function by its Lipschitz
constant or modulus of continuity, the complexity of a
piecewise affine function can be thought of as the minimal 
number of affine pieces needed to define it.  

\begin{theorem}\label{T:max-affine-approx}
Let $d\geq 1$ and $f:[0,1]^d\gives \R_+$ be the function computed by some 
$\Relu$ net with input dimension $d$, output dimension $1,$ and
arbitrary width. There
exist affine functions $g_\alpha, h_\beta:[0,1]^d\gives \R$  
such that $f$ can be written as the difference of positive convex functions:
\begin{equation}\label{E:f-def1}
f=g- h,\qquad\qquad g:=\max_{1\leq \alpha\leq N}
g_\alpha,\qquad h:= \max_{1\leq\beta\leq M} h_\beta.
\end{equation}
Moreover, there exists a feed-forward neural net $\mathcal N$ 
with ReLU activations, input dimension $d,$ hidden 
  layer width $d+3,$ output dimension $1,$ and 
  \begin{equation}\label{E:convex-depth}
\text{depth}\lr{\mathcal N}=2(M+N)
\end{equation}
 that computes $f$ exactly. Finally, if $f$ is convex (and hence $h$
 vanishes), then the width of $
 \mathcal N$ can be taken to be $d+1$ and the depth can be taken to $N.$
\end{theorem}
\noindent The fact that the function computed by a $\Relu$ net can be
written as \eqref{E:f-def1} follows from Theorem 2.1 in
\cite{ABM}. The novelty in Theorem \ref{T:max-affine-approx} is
therefore the uniform width estimate $d+3$ in the
representation on any function computed by a $\Relu$ net and the $d+1$
width estimate for convex functions. Theorem \ref{T:max-affine-approx} will be used in the
proof of Theorem \ref{T:lip-approx}. 

\section{Relation to Previous Work}

This article is related to several strands of prior work:
\begin{enumerate}
\item Theorems \ref{T:lip-approx}-\ref{T:max-affine-approx} are ``deep and narrow'' analogs of
 the well-known ``shallow and wide'' universal approximation results (e.g. Cybenko
\cite{Cyb} and Hornik- Stinchcombe -White \cite{HSW}) for feed-forward neural
nets. Those articles show that 
essentially any scalar function $f:[0,1]^d\gives \R$ on the
$d-$dimensional unit cube can be arbitrarily well-approximated by a
feed-forward neural net with a single hidden layer with arbitrary
width. Such results hold for a wide class of nonlinear activations but
are not particularly illuminating from the point of understanding the
expressive advantages of depth in neural nets.
\medskip
\item The results in this article complement the work of Liao-Mhaskar-Poggio
\cite{LMP} and Mhaskar-Poggio \cite{MhPo}, who consider the advantages
of depth for representing certain heirarchical or compositional
functions by neural nets with both $\Relu$ and non-$\Relu$
activations. Their results (e.g. Theorem 1
in \cite{LMP} and Theorem 3.1 in \cite{MhPo}) give bounds on the width
for approximation both for shallow and certain deep heirarchical nets. 
\medskip
\item Theorems \ref{T:lip-approx}-\ref{T:max-affine-approx} are also
quantitative analogs of Corollary 2.2 and Theorem 2.4 
  in the work of Arora-Basu-Mianjy-Mukerjee \cite{ABM}. Their results
  give bounds on the depth of a $\Relu$ net needed to compute exactly
  a piecewise linear function of $d$ variables. However, except when
  $d=1,$ they do not obtain an estimate on the number of neurons in
  such a network and hence cannot bound the width of the hidden
  layers. 
\medskip
\item  Our results are related to Theorems II.1 and II.4 of
  Rolnick-Tegmark \cite{RoTe}, which are themselves extensions of
  Lin-Rolnick-Tegmark \cite{LRT}. Their results give lower bounds on the total
  size (number of neurons) of a neural net (with non-$\Relu$
  activations) that approximates sparse 
  multivariable polynomials. Their bounds do not imply a control on
  the width of such networks that depends only on the number of
  variables, however. 
\medskip
\item This work was inpsired in part by questions raised in the work of
  Telgarsky \cite{Tel1,Tel2, Tel3}. In particular, in Theorems 1.1 and 1.2
  of \cite{Tel1}, Telgarsky constructs
  interesting examples of sawthooth functions that can be computed
  efficiently by deep width $2$ $\Relu$ nets that
  cannot be well-approximated by shallower networks with a simlar number
  of parameters. 
\medskip
\item Theorems \ref{T:lip-approx}-\ref{T:max-affine-approx} are quantitative statements about the
  expressive power of depth without the aid of width. This topic, usually without
  considering bounds on the width, has been taken up by many
  authors. We refer the reader to \cite{PLR, RPK} for several interesting
  quantitative measures of the complexity of functions computed by
  deep neural nets. 
\medskip
\item Finally, we refer the reader to the interesting work of Yarofsky
  \cite{Yar}, which provides bounds on the total number of
  parameters in a $\Relu$ net needed to approximate a given class of
  functions (mainly balls in various Sobolev spaces). 

\end{enumerate}

\section{Acknowledgements}
It is a pleasure to thank Elchanan Mossel and Leonid Hanin for many helpful
discussions. This paper originated while I attended EM's class on deep
learning \cite{EM-class}. In
particular, I would like to thank him for suggesting proving
quantitative bounds in Theorem \ref{T:max-affine-approx} and for
suggesting that a lower bound can be obtained by taking piece-wise
linear functions with many different directions. He also pointed out
that the width estimates for continuous function in Theorem
\ref{T:lip-approx} where sub-optimal in a previous draft. l would also like to
thank Leonid Hanin for detailed comments on a several previous drafts and for
useful references to results in approximation theory. I am also grateful to
Brandon Rule and Matus Telgarsky for comments on an earlier version 
of this article. I am also grateful to BR for the original suggestion to investigate the
expressivity of neural nets of width $2$. I also would like to
thank Max Kleiman-Weiner for useful comments and discussion. Finally,
I thank Zhou Lu for pointing out a serious error what used to be Theorem 3 in a previous
version of this article. I have removed that result.

\section{Proof of Theorem \ref{T:max-affine-approx}}
\noindent We first treat the case
\[f=\sup_{1\leq \alpha\leq N} g_\alpha, \qquad g_\alpha:[0,1]^d\gives \R\quad
\text{affine}\]
when $f$ is convex. We seek to show that $f$ can be exactly represented by a $\Relu$ net
with input dimension $d,$ hidden layer width $d+1$, and depth $N.$ Our
proof relies on the following observation.  

\begin{lemma}\label{L:key}
  Fix $d\geq 1,$ let $T:\R_+^d\gives \R$ be an arbitrary function,
  and $L:\R^d\gives \R$ be affine. Define an invertible affine transformation
  $A:\R^{d+1}\gives\R^{d+1}$ by
\[A(x,y)= \lr{x, L(x)+y}.\]
Then the image of the graph of $T$ under 
\[A\circ \Relu \circ A^{-1}\]
is the graph of $x\mapsto \max\set{T(x) ,L(x)},$ viewed as a function
on $\R_+^d.$
\end{lemma}
\begin{proof}
  We have $A^{-1}(x,y)=(x, -L(x)+y).$ Hence, for each $x\in \R_+^d,$
  we have
  \begin{align*}
A\circ \Relu \circ A^{-1}(x,T(x))&=\lr{x, \lr{T(x)-L(x)}{\bf
    1}_{\{T(x)-L(x)>0\}} + L(x)}\\
&=\lr{x, \max\set{T(x),L(x)}}.
  \end{align*}
\end{proof}
\noindent We now construct a neural net that computes $f.$ Define invertible
affine functions $A_\alpha:\R^{d+1}\gives \R^{d+1}$ by  
\[A_\alpha(x,x_{d+1}):=\lr{x,g_\alpha(x)+x_{d+1}},\qquad
x=(x_1,\ldots, x_d),\]
and set
\[H_\alpha:=A_\alpha \circ \Relu \circ A_\alpha^{-1}.\]
Further, define 
\begin{equation}
  H_{\text{out}}:=\Relu \circ \inprod{\vec{e}_{d+1}}{\cdot}
\end{equation}
where $\vec{e}_{d+1}$ is the $(d+1)-$st standard basis vector so that
$\inprod{\vec{e}_{d+1}}{\cdot} $ is the linear map from $\R^{d+1}$ to $\R$  
that maps $(x_1,\ldots, x_{d+1})$ to $x_{d+1}.$ Finally, set
\[H_{\text{in}}:= \Relu\circ \lr{\text{id}, 0},\]
where $\lr{\text{id}, 0}(x)=(x,0)$ maps $[0,1]^d$ to the graph of the
zero function. Note that the $\Relu$
in this initial layer is linear. With this notation, repeatedly using
Lemma \ref{L:key}, we find that
\[H_{\text{out}}\circ H_N\circ\cdots \circ H_1\circ H_{\text{in}}\]
therefore has input dimension $d,$ hidden layer width $d+1,$ depth
$N$ and computes $f$ exactly. \medskip

\noindent Next, consider the general case when $f$ is given by
\[f=g-h,\qquad g= \sup_{1\leq \alpha\leq N}g_\alpha ,\qquad h=
\sup_{1\leq \beta \leq M} h_\beta\]
as in \eqref{E:f-def1}. For this situation, we use a different way of
computing the maximum using $\Relu$ nets. 
\begin{lemma}\label{L:max}
  There exists a $\Relu$ net $\mathcal M$ with input dimension $2,$ hidden layer
  width $2$, output dimension $1$ and depth $2$ such that 
\[\mathcal M\lr{x,y}=\max\set{x,y},\qquad x\in \R, y \in \R_+.\]
\end{lemma}
\begin{proof}
Set $A_1(x,y):=(x-y, y),\, A_2(z,w)=z+w,$ and define 
\[\mathcal M=\Relu \circ A_2 \circ \Relu \circ A_1.\]
We have for each $y\geq 0, x\in \R$
\[f_{\mathcal M}(x,y)=\Relu((x-y){\bf 1}_{\set{x-y>0}} +
y)=\max\set{x,y},\]
as desired. 
\end{proof}

We now describe how to construct a $\Relu$ net $\mathcal N$ with input dimension $d$,
hidden layer width $d+3,$ output dimension $1,$ and depth $2(M+N)$ that
exactly computes $f$. We use width $d$ to copy the
input $x$, width $2$ to compute successive maximums of the positive affine
functions $g_\alpha, h_\beta$ using the net $\mathcal M$ from Lemma
\ref{L:max} above, and
width $1$ as memory in which we store $g=\sup_\alpha g_\alpha$ while
computing $h=\sup_\beta h_\beta.$ The final layer computes the
difference $f=g-h.$ \qed

\section{Proof of Theorem \ref{T:lip-approx}}\label{S:lip-approx}
\noindent We begin by showing \eqref{E:conv-min} and
\eqref{E:lip-depth-approx}. Suppose $f:[0,1]^d\gives \R_+$ is convex
and fix $\ep>0.$ A simple discretization argument shows that there
exists a piecewise affine convex function $g:[0,1]^d\gives \R_+$ such
that $\norm{f-g}_{C^0}\leq \ep.$ By Theorem \ref{T:max-affine-approx},
$g$ can be a exactly represented by a $\Relu$ net with hidden layer width
$d+1.$ This proves \eqref{E:conv-min}. In the case that $f$ is
Lipschitz, we use the following, a special case of Lemma 4.1 in
\cite{BGS}. 
\begin{proposition}\label{P:max-affine-approx}
  Suppose $f:[0,1]^d\gives \R$ is convex and Lipschitz with Lipschitz
  constant $L$. Then for every $k\geq 1$ there exist $k$ affine maps
  $A_j:[0,1]^d\gives \R$ such that 
\[\norm{f-\sup_{1\leq j \leq k} A_j}_{C^0}\leq 72 L \,d^{3/2} k^{-2/d}.\]
\end{proposition}
\noindent Combining this result with Theorem \ref{T:max-affine-approx} proves
\eqref{E:lip-depth-approx}. We turn to checking \eqref{E:cont-approx}
and \eqref{E:smooth-approx}. We need the following observations, which
seems to be well-known but not written down in the literature. 

\begin{lemma}\label{L:side}
  Let $\mathcal N$ be a $\Relu$ net with input dimension $d,$ a single 
hidden layer of width $n,$ and output dimension $1.$ There exists another
$\Relu$ net $\twiddle{\mathcal N}$ that computes the same function as
$\mathcal N$ but has input dimension $d$ and $n+2$
hidden layers with width $d+2.$ 
\end{lemma}
\begin{proof}
  Denote by $\set{A_j}_{j=1}^n$ the affine functions computed
by each neuron in the hidden layer of $\mathcal N$ so that  
\[f_{\mathcal N}(x)=\Relu\lr{b + \sum_{j=1}^n c_j\Relu(A_j(x))}.\]
Let $T>0$ be sufficiently large that 
\[T+ \sum_{j=1}^k c_j \Relu(A_j(x))>0,\qquad \forall 1\leq k \leq
n,~~x\in [0,1]^d.\]
The affine transformations $\twiddle{\mathcal A}_j$ computed by the
$j^{th}$ hidden layer of $\twiddle{\mathcal N}$ are then
\[\twiddle{A}_1(x):=\lr{x,  A_j(x),T} \qquad\text{and}\qquad
\twiddle{A}_{n+2}(x, y, z) =  z-T+b,\qquad x\in \R^d,\, y,z\in \R\]
and
\[\twiddle{A}_j(x,y,z)=\lr{x, A_j(x), z+ c_{j-1}y },\qquad j=2,\ldots, n+1.\]
We are essentially using width $d$ to copy in the input variable,
width $1$ to compute each $A_j$ and width $1$ to store the
output.
\end{proof}

Recall that positive continuous functions can be
arbitrarily well-approximated by smooth functions and hence by $\Relu$
nets with a single hidden layer (see 
e.g. Theorem 3.1 \cite{MhPo}). The relation \eqref{E:cont-approx} therefore
follows from Lemma \ref{L:side}. Similarly, by Theorem
3.1 in \cite{MhPo}, if $f$ is smooth, then there exists 
$K=K(d)>0$ and a constant $C_f$ depending only on the maximum value of
the first $K$ derivatives of $f$ such that 
\[\inf_{\mathcal N}\norm{f-f_{\mathcal N}}\leq C_f n^{-1/d},\]
where the infimum is over $\Relu$ nets $\mathcal N$ with a single
hidden layer of width $n$. Combining this with Lemma \ref{L:side}
proves \eqref{E:smooth-approx}.  

It remains to prove \eqref{E:lip-depth} and \eqref{E:lip-approx}. To
do this, fix a positive continuous function $f:[0,1]^d\gives \R_+$
with modulus of continuity $\w_f.$ Recall that the volume of the unit
$d$-simplex is $1/d!$ and fix $\ep>0.$ Consider the partition 
\[[0,1]^d=\bigcup_{j=1}^{d!/\w_f(\ep)^d} \mathcal P_j\]
of $[0,1]^d$ into $d!/ \w_f(\ep)^d$ copies of $\w_f(\ep)$ times the standard
$d$-simplex. Define $f_\ep$ to be a piecewise linear approximation to
$f$ obtained by setting $f_\ep$ equal to $f$ on the vertices of the
$\mathcal P_j$'s and taking $f_\ep$ to be affine on their
interiors. Since the diameter of each $\mathcal P_j$ is $\w_f(\ep),$ we have
\[\norm{f-f_\ep}_{C^0}\leq \ep.\]
Next, since $f_\ep$ is a piecewise affine function, by Theorem 2.1 in
\cite{ABM} (see Theorem \ref{T:max-affine-approx}), we may write
\[f_\ep=g_\ep-h_\ep,\]
where $g_\ep,h_\ep$ are convex, positive, and piecewise
affine. Applying Theorem \ref{T:max-affine-approx} completes the proof
of \eqref{E:lip-depth} and \eqref{E:lip-approx}.
\qed

\end{document}